\documentclass[twoside,11pt]{article}

\usepackage{jmlr2e}
\usepackage{amsmath}
\usepackage{bm}
\usepackage{booktabs}
\usepackage{multirow}
\usepackage{microtype}
\usepackage{enumitem}
\usepackage{lastpage}
\usepackage{float}

\newtheorem{assumption}[theorem]{Assumption}

\newcommand{\E}{\mathbb{E}}
\newcommand{\Prob}{\mathbb{P}}

\newcommand{\ind}[1]{\mathbf{1}\{#1\}}
\newcommand{\eps}{\varepsilon}
\newcommand{\kl}{\mathrm{kl}}
\newcommand{\KL}{\mathrm{KL}}

\newcommand{\filt}{\mathcal{F}}
\newcommand{\Lb}{L_{\delta}}
\newcommand{\discern}{\textsc{Discern}}

\jmlrheading{}{2026}{}{}{}{}{Vishnu Bindu Balachandran}
\ShortHeadings{Pay Only for Disagreement}{Vishnu Bindu Balachandran}
\firstpageno{1}

\begin{document}

\title{Pay Only for Disagreement: Certified No-Regression Verdicts for
Model Updates with Matching Label-Complexity Bounds}

\author{\name Vishnu Bindu Balachandran
        \email vishnubindubalachandran@outlook.com \\
       \addr Independent Researcher \\
       \addr www.vishnubindubalachandran.com}

\editor{}

\maketitle

\begin{abstract}%
Every production model is updated, by retraining, fine-tuning, quantization,
or a silent vendor swap, and each update risks being worse than what it
replaced. We formalize update promotion as certified paired risk-difference
auditing. Our starting point is a support identity. The risk difference
between two models lives on the inputs where they disagree, observable
without labels. We build \discern{}, a sequential two-tier protocol. A
zero-label tier certifies benign updates whose disagreement rate is below
tolerance from unlabeled traffic alone. An audited tier labels only sampled
disagreements through an anytime-valid confidence sequence, valid at every
stopping time and under any label-routing rule, even an adversarial judge. We
prove finite-sample validity and matching label-complexity bounds
$\widetilde{\Theta}(\rho^2/\eps^2)$ at the rate level, so exploiting free
disagreement provably saves a factor $\Theta(1/\rho)$ over any pairing-blind
auditor, and the guarantee composes across an unbounded sequence of
promotions from one error budget. Across $14{,}000$+ replayed audit streams
over $785$ update pairs, including LoRA fine-tunes of language models up to
1.4B parameters, miscoverage is $0.0002$ (nominal $5\%$), power $0.986$ with
zero false alarms, and $56\%$ of benign updates certify with zero labels.
Each audit emits a machine-checkable evidence record for EU AI Act
post-market monitoring.
\end{abstract}


\section{Introduction}
\label{sec:intro}

Machine learning systems in production are not static artifacts. Teams retrain
on fresh data, adjust hyperparameters, fine-tune on new domains, quantize
weights for cheaper inference, distill large models into small ones, and
absorb silent upstream changes when a managed API provider swaps its backbone.
Each such event replaces a deployed model $f$ with a candidate $g$, and each
replacement carries the same operational risk, since the update may quietly
be worse than what it replaced, either overall or on a slice of traffic that
matters.
Post-deployment incident reports and the model-update literature on churn and
backward compatibility document this failure mode repeatedly
\citep{fard2016launch,yan2021positive,srivastava2020empirical}.

Despite the ubiquity of the event, the promotion decision is statistically
informal in most organizations. The dominant practices are (i) fixed offline
test suites, which cannot reflect current traffic and are exhausted by
repeated use, (ii) A/B tests on business metrics, which are slow, costly in
exposure, and do not certify model correctness, and (iii) uniform labeling
campaigns on sampled traffic, which are statistically sound but pay for an
enormous number of labels that carry no information about the update. The
third point is the crux. If the old and new model agree on a traffic point,
that point contributes exactly zero to the risk difference between them, no
matter what the true label is. Labels spent on agreements are wasted, and
under realistic updates the models agree on $80\%$ to $99.9\%$ of traffic.

This paper treats the promotion decision as a first-class statistical object
with certified verdicts, matching label-complexity bounds, and lifetime
soundness across a sequence of updates. We study the paired risk difference
$\Delta = \E[\ell(g(X),Y)] - \E[\ell(f(X),Y)]$
between candidate $g$ and incumbent $f$, signed so that $\Delta > 0$ means
the update is a regression. The operator wants one of two certified
verdicts at tolerance $\eps$ and confidence $1-\delta$, either a no-regression
certificate asserting $\Delta < \eps$, or a regression alarm asserting
$\Delta > 0$. The audit must run on live traffic, must remain valid when the
operator peeks continuously and stops as soon as evidence suffices, and must
spend as few human labels as possible.

\paragraph{The support identity and its consequences.}
Everything follows from one elementary observation. Write
$A = \ind{f(X)\neq g(X)}$ for the disagreement indicator, observable for free
because both models already score the traffic. For any loss that depends on
the input only through the prediction, agreement forces identical losses, so
the paired difference $D$ vanishes whenever $A=0$ and
\begin{equation}
\label{eq:support}
\Delta \;=\; \E[D\,A],
\qquad
|\Delta| \;\le\; B\,\rho,
\qquad \rho := \Prob(A=1),
\end{equation}
where $B$ bounds the loss range. Identity~\eqref{eq:support} has two
consequences that we develop into a complete theory. First, if unlabeled
traffic certifies that $\rho \le \eps/B$, then $|\Delta|\le\eps$ is certified
with zero labels. Quantized builds, routine retrains, and hyperparameter
touches typically live in this regime, and in our experiments $98\%$ of such
benign updates certify label-free. Second, when $\rho$ is not negligible, all
information about $\Delta$ is concentrated on disagreements, so an auditor
should label only there. Figure~\ref{fig:overview} overviews the resulting
protocol. We make this precise with an importance-weighted
estimator wrapped in an empirical-Bernstein confidence sequence, and we prove
it is not merely a good idea but a rate-optimal one. Its label complexity matches
an information-theoretic lower bound up to logarithmic factors, and the gap to any pairing-blind
procedure is exactly the disagreement rate itself.

A scope note belongs up front. The support identity requires the loss to
depend on the input only through the prediction (Assumption~\ref{ass:pred}),
which holds for the standard supervised losses and for the classification
fine-tunes in our experiments, but excludes losses that read model scores or
a judge's verdict, including calibration and preference or reward losses. Our
certified guarantees are for that prediction-determined class. A soft
generalization to score-reading losses that are Lipschitz in the model
output, which reaches calibration and smooth reward objectives, is developed
in Section~\ref{sec:soft}; open-ended judged generation remains outside the
theory and we do not claim it.

\begin{figure}[t]
\centering
\includegraphics[width=\textwidth]{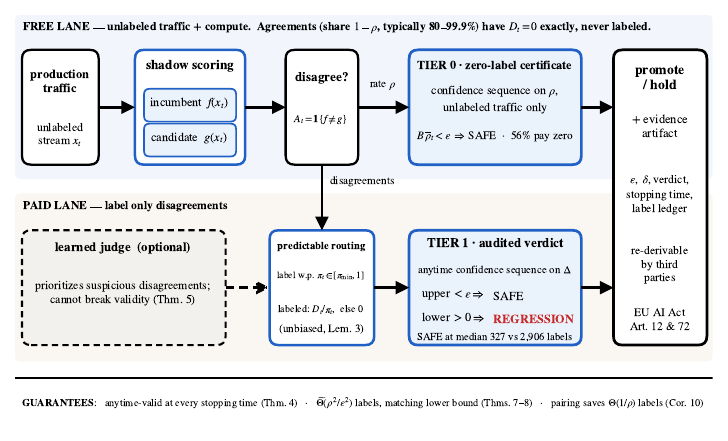}
\caption{The \discern{} protocol, organized by what it spends. In the free
lane, both models shadow-score unlabeled traffic and the disagreement
indicator stream (both branches of the gate) drives a zero-label
confidence sequence on $\rho$, and when $B\overline\rho_t<\eps$ the update
is certified safe without a single label (Tier 0). Only sampled
disagreements enter the paid lane, through a predictable routing rule with
importance-weighted increments, and an anytime-valid confidence sequence
on $\Delta$ issues the verdict (Tier 1). Every quantity shown is measured
in Section~\ref{sec:experiments}. In brief, $56\%$ of benign audits certify with
zero labels, the audited band pays a median of $327$ versus $2{,}906$
labels for uniform labeling, and injected regressions alarm at a median of
$964$ stream points. The footer lists the paper's formal guarantees.}
\label{fig:overview}
\end{figure}

\paragraph{Contributions.}
\begin{enumerate}[leftmargin=1.6em,itemsep=1pt,topsep=2pt]
\item \textbf{Problem formalization.} We cast update promotion as certified
paired risk-difference auditing with anytime-valid semantics
(Section~\ref{sec:setting}), an object distinct from single-model risk
certification, two-sample testing, and offline evaluation. A soft support
identity (Proposition~\ref{prop:soft}) extends the whole protocol from
prediction-determined losses to losses Lipschitz in the model output, such
as calibration and smooth reward losses.
\item \textbf{The \discern{} protocol.} A two-tier sequential audit
(Section~\ref{sec:protocol}) with a zero-label certificate from unlabeled
disagreement monitoring, and an audited tier that labels sampled
disagreements under any predictable routing rule, with per-slice
simultaneous verdicts and family-wise soundness across a whole sequence
of promotions (Proposition~\ref{prop:sequence}).
\item \textbf{Validity theory.} Time-uniform $(1-\delta)$ coverage of the
audited estimate at every data-dependent stopping time, for every predictable
routing rule, judge-driven or adversarial (Theorem~\ref{thm:validity},
Theorem~\ref{thm:routing}).
\item \textbf{Matching label-complexity bounds.} An upper bound of
$\widetilde O\!\big(B^2\rho^2/\eps^2 + B\rho/\eps\big)$ labels to an
$(\eps,\delta)$ verdict (Theorem~\ref{thm:upper}) and a change-of-measure
lower bound $\Omega(\rho^2\eps^{-2}\log(1/\delta))$ for any adaptive
labeling-and-stopping procedure (Theorem~\ref{thm:lower}), matching up to
logarithmic factors. The paired comparison object goes back at least to
\citet{sawade2012active}, and certified update approval to
\citet{feng2021approval}. What is new here, to our knowledge, is the
characterization of the label cost of a certified update verdict as a
function of the disagreement rate, with matching bounds.
\item \textbf{The value of pairing.} Any procedure that cannot see the
disagreement signal needs $\Omega(\rho\,\eps^{-2}\log(1/\delta))$ labels
(Theorem~\ref{thm:blind}), and uniform labeling achieves it. Pairing
therefore buys a factor $\Theta(\rho)$, a compute-for-labels exchange in which
running both models on unlabeled traffic, which is cheap, eliminates a
$1/\rho$ multiple of human labeling, which is expensive.
\item \textbf{Large-scale validation.} A pre-registered evaluation over
$715$ feature-space update pairs and $6{,}890$ replayed audit streams across
four feature modalities and eleven update types, plus $70$ real LoRA
fine-tune update pairs ($140$ further streams) of Pythia language models,
with a held-out confirmation on a reserved seed family
(Section~\ref{sec:experiments}). Empirical miscoverage is
$0.0002$ (nominal $5\%$), power is $0.986$ with zero false alarms, and label
spend follows the theoretical law with $R^2 = 0.999$. Head-to-head
comparisons against sequential adaptations of prediction-powered and
active-testing baselines on identical streams
(Section~\ref{sec:exp-baselines}), and a windowed variant that repairs
calibration under adversarial drift (Section~\ref{sec:exp-drift}), complete
the empirical case.
\end{enumerate}

\paragraph{Why this is the right target for a certified method.}
A certificate is only as valuable as the decision it unlocks. The decision
here is concrete and recurring, namely whether to promote or hold an update.
Regulation is converging on the same object. Article 72 of the EU AI Act requires providers
of high-risk systems to run post-market monitoring proportionate to the risk,
and logging obligations under Article 12 require traceable evidence for
system changes. A \discern{} audit emits, per promotion, a machine-checkable
tuple (tolerance, confidence, verdict, label spend, stopping time, per-slice
outcomes) that serves directly as such evidence. No comparable
finite-sample instrument exists in the update-monitoring literature, which
either assumes labeled evaluation sets, provides asymptotic or heuristic
monitoring, or certifies single-model properties rather than the paired
difference that the promotion decision actually consumes
(Section~\ref{sec:related}).

\section{Related Work}
\label{sec:related}

\paragraph{Paired comparison of classifiers.}
The observation that paired accuracy comparison hinges on discordant pairs
goes back to \citet{mcnemar1947note}, and sequential treatments of two-sample
testing are classical \citep{wald1945sequential}. Restricting adjudication to
discordant pairs to save labels recurs in fixed-sample model evaluation,
where \citet{musgrove2023discordant} report over a $90\%$ reduction in
adjudications, but under an assumed class prevalence and with neither a
sequential guarantee nor a lower bound. Our contribution is not the
identity itself but its consequences under modern constraints, namely anytime
validity under continuous monitoring, adaptive label routing with
importance-weighting, matching label-complexity bounds, and a zero-label
tier. Recent resolution diagnostics for paired evaluation
\citep{kotawala2026resolution} quantify whether a fixed-$n$ comparison has
enough labeled points to resolve, but assume the labels already exist. We
optimize and certify the labeling process itself.

\paragraph{Anytime-valid inference.}
Our confidence sequences build on time-uniform boundaries from
nonnegative-supermartingale arguments \citep{ville1939etude,robbins1970},
modern empirical-Bernstein confidence sequences
\citep{howard2021time,waudby2024estimating}, and the game-theoretic testing
programme \citep{ramdas2023game}. Betting-style audits have been used for
fairness properties of a fixed model \citep{chugg2023auditing} and adaptive
audits of a single AI system with anytime-valid guarantees
\citep{zhou2026adaptive}. Sequential drift attribution for LLM evaluation
pipelines \citep{li2026drifted} monitors a judged pipeline over time. Closest
in mechanism, \citet{karampatziakis2021offpolicy} build off-policy confidence
sequences with importance weights inside a betting martingale and apply them
to gated deployment of a contextual-bandit system, and
\citet{xu2024active} construct anytime-valid risk-controlling prediction sets
under active labeling with a query budget. These share our core machinery,
Horvitz--Thompson weights within an anytime-valid betting confidence
sequence under adaptive querying. None of these anytime-valid audits target
the paired update delta, none exploit the disagreement support to drive a
zero-label tier, and none characterize label complexity or prove a matching
lower bound. We use the
confidence-sequence machinery as a component and prove rate-matching
complexity results for the promotion problem.

\paragraph{Model comparison and update approval.}
Two prior lines come closest to our object. \citet{sawade2012active} compare the risks of two given models on a fixed labeling budget by sampling test instances from a variance-optimal instrumental distribution, the paired object with active, importance-weighted labeling. Their guarantee is an asymptotic normal-approximation test at a fixed horizon. We add what deployment needs and what their analysis does not provide, namely finite-sample anytime-valid verdicts, an explicit label-complexity characterization with a matching lower bound, and a zero-label tier. Notably, their optimal instrumental distribution already concentrates on informative points, and our Theorem~\ref{thm:lower} can be read as the information-theoretic reason why. Separately, \citet{feng2021approval} and \citet{feng2021learning} study approval policies for modifications to deployed ML systems as an online hypothesis testing problem with error control across a sequence of approvals, including the bio-creep phenomenon. That line certifies sequences of approvals on labeled test batches. We certify each promotion from unlabeled traffic at optimal label cost, our audits compose across a sequence of promotions with family-wise control of bio-creep (Proposition~\ref{prop:sequence}), and their alpha-investing refinement ports directly. \citet{podkopaev2022tracking} track the risk of a single deployed model with anytime-valid guarantees, and \citet{shanmugam2025evaluating} estimate metrics for multiple models from labeled and unlabeled data without certified verdicts. Neither targets the paired promotion decision. Concurrent with this work, \citet{shawn2026pace} propose a McNemar-style anytime-valid commit gate for self-evolving agents, using a testing-by-betting e-process on paired instances to control the false-commit rate. That design shares our anytime-valid, paired, evidence-on-disagreements structure, and reaches similar conclusions about uncontrolled greedy acceptance, but provides neither a zero-label certificate, nor a label-complexity characterization, nor a matching lower bound, which are the central objects here.

\paragraph{Label-efficient evaluation.}
Active testing and model-assisted evaluation reduce labels for estimating a
single model's risk \citep{kossen2021active,angelopoulos2023prediction}, with
sequential and active variants of two-sample testing in
\citet{li2023active,shekhar2501label}. Prediction-powered and surrogate-based
estimators use model outputs to sharpen single-quantity inference
\citep{angelopoulos2023prediction,zrnic2024active}, and
\citet{boyeau2025autoeval} apply prediction-powered inference to pairwise
model comparison from synthetic labels. These are fixed-horizon estimators
of a metric, not certified anytime-valid promotion verdicts, and none
provides a label-complexity lower bound. The estimand differs in a way that
changes the complexity theory. For the paired difference, an observable
signal (disagreement) supports the entire estimand, which makes the optimal
label complexity quadratic rather than linear in $\rho$ and creates the
$\Theta(\rho)$ separation we prove. Cross-model disagreement has been
proposed as a label-free correctness heuristic \citep{gorbett2026cross} and
disagreement-aware aggregate evaluation appears in \citet{bonagiri2026stableval}.
These works use disagreement as a signal for prediction quality, not as the
support of a certified estimand, and provide no finite-sample guarantees.

\paragraph{Disagreement-based active learning.}
The principle that labels are informative only where competing hypotheses
disagree is the foundation of disagreement-based active learning, from the
query-by-disagreement of \citet{cohn1994improving} to the agnostic $A^2$
algorithm of \citet{balcan2009agnostic} and its analysis through Hanneke's
disagreement coefficient \citep{hanneke2007bound,hanneke2014theory}. That
theory bounds the labels needed to \emph{learn} a low-error classifier by
the rate at which the disagreement region of a shrinking version space
contracts around the target hypothesis. Our problem differs in both object
and mathematics. We learn no classifier and maintain no version space. We
hold two fixed given models and certify a scalar functional, the paired
risk difference, with an anytime-valid guarantee. The disagreement region
is fixed at the observable rate $\rho$ rather than shrinking with a
disagreement coefficient, and our tolerance $\eps$ bounds a risk difference
rather than the excess error of a learned hypothesis. Consequently our
$\widetilde\Theta(\rho^2/\eps^2)$ characterization is a mean-estimation and
sequential-testing complexity for a functional, proved by a
change-of-measure construction (Theorem~\ref{thm:lower}), not a
version-space label-complexity bound, and neither result subsumes the
other. What the two share is the operational core we take as our starting
point, that information concentrates on disagreements. The bound
$|\Delta|\le B\rho$ of Lemma~\ref{lem:support} is elementary and is the
evaluation-side counterpart of the classical disagreement-region bound. The
idea of using model disagreement on unlabeled data to validate and compare
classifiers is older still, going back to the co-validation of
\citet{madani2004covalidation}, who showed that disagreement lower-bounds
error and can compare models without labels, and to the unlabeled-data
generalization bounds of \citet{kaariainen2005generalization}, which bound a
risk gap by an unlabeled disagreement probability. Concurrent work of
\citet{bazinet2026bound} likewise certifies the gap between two predictors'
risks by evaluating a disagreement bound on unlabeled data, for static
generalization certificates rather than a sequential update verdict. Our
contribution is not this bound but its sequential certification, the
zero-label tier it enables, and the matching two-sided label-complexity
theory built around it.

\paragraph{Update churn and backward compatibility.}
The applied literature documents the operational pain of updates, from
prediction churn between successive models \citep{fard2016launch} to negative
flips and backward-compatibility metrics \citep{yan2021positive} and
regression testing practice for ML systems \citep{srivastava2020empirical}.
These works measure or reduce disagreement. We complete the pipeline. Given
that updates disagree, we certify, at optimal label cost, whether the
disagreement hurts.

\paragraph{Model multiplicity.}
Predictive multiplicity \citep{marx2020predictive,black2022model} studies
distinct models with similar average risk. The disagreement region between
near-equivalent models is exactly where our audits concentrate, and
multiplicity explains why the benign-update regime ($\Delta \approx 0$,
$\rho$ moderate) is common and why certifying it cheaply matters.

\section{Setting and the Support Identity}
\label{sec:setting}

Let $(X_t, Y_t)_{t\ge1}$ be an i.i.d.\ stream from an unknown distribution
$P$ over $\mathcal X\times\mathcal Y$ (drift is treated in
Section~\ref{sec:drift}). An incumbent model $f$ and a candidate $g$ map
$\mathcal X$ to a prediction space $\mathcal A$, and a loss
$\ell : \mathcal A\times\mathcal Y \to [0, B]$ scores predictions. Define the
paired difference and its mean
\begin{equation}
D_t \;=\; \ell\big(g(X_t), Y_t\big) - \ell\big(f(X_t), Y_t\big) \in [-B, B],
\qquad
\Delta \;=\; \E[D_t],
\end{equation}
so $\Delta > 0$ means the candidate is worse (a regression) and $\Delta \le 0$
means it is no worse. For $0/1$ loss, $\Delta$ is the accuracy drop and
$B=1$. The auditor observes both predictions on every point for free, hence
also the disagreement indicator $A_t = \ind{f(X_t)\neq g(X_t)}$ with rate
$\rho = \Prob(A_t = 1)$, and may request the label $Y_t$ at unit cost.

\begin{assumption}[Prediction-determined loss]
\label{ass:pred}
The loss is a function of the prediction and label alone, $\ell(a,y)$, with
no other dependence on the input. Its operative content is that whenever the
two models make the same prediction, $f(X_t)=g(X_t)$, they incur the same
loss, so $D_t=0$.
\end{assumption}

Assumption~\ref{ass:pred} holds for every standard supervised loss ($0/1$,
cost-sensitive misclassification, top-$k$, bounded regression losses on
point predictions, exact-match and slot-level scores for structured
prediction). It fails only when the loss consults model internals, for
example calibration losses on confidence scores. Every guarantee in this
paper is for this deterministic-prediction class. Judged and
preference-based losses, including chat-quality and RLHF objectives, are
outside it, and our language-model experiments are classification
fine-tunes that satisfy Assumption~\ref{ass:pred}
(Section~\ref{sec:zoo}). We discuss the extension in
Section~\ref{sec:limitations}.

\begin{lemma}[Support identity]
\label{lem:support}
Under Assumption~\ref{ass:pred},
$\Delta = \E[D_t A_t]$ and $|\Delta| \le B\rho_1 \le B\rho$, where
$\rho_1 = \Prob(D_t \neq 0)$.
\end{lemma}

\begin{proof}
$D_t = D_t A_t$ pointwise since $A_t = 0$ forces $D_t = 0$. Take
expectations, then bound $|D_t|\le B$ on the event $\{D_t \ne 0\}
\subseteq \{A_t = 1\}$.
\end{proof}

The identity is elementary, and the bound $|\Delta|\le B\rho$ is the
evaluation-side form of the disagreement-region principle that underlies
disagreement-based active learning
\citep{hanneke2014theory}. What is not classical, and what the rest of the
paper develops, is its sequential certification from unlabeled traffic
and the matching label-complexity theory it implies
(Section~\ref{sec:related} contrasts the two settings).
The refinement through $\rho_1$ matters in multiclass problems where both
models can disagree yet both be wrong, contributing $A_t = 1$ but $D_t = 0$.
All upper bounds below can be read with $\rho_1$ in place of $\rho$ wherever
variance is concerned. We state them with $\rho$ for simplicity since
$\rho_1\le\rho$ and only $\rho$ is observable without labels.

\paragraph{Verdicts.}
Fix a tolerance $\eps\in(0,B)$ and confidence level $\delta\in(0,1)$. At any
data-dependent time the auditor may issue one of two verdicts.
\begin{itemize}[leftmargin=1.6em,itemsep=0pt,topsep=2pt]
\item \textsc{Safe}$(\eps)$, which asserts $\Delta < \eps$ (promote the update).
\item \textsc{Regression}, which asserts $\Delta > 0$ (block or roll back).
\end{itemize}
\textsc{Safe}$(\eps)$ is a tolerance statement in the spirit of noninferiority testing. It asserts $\Delta<\eps$ rather than $\Delta\le 0$, and we use no-regression in this $\eps$-tolerance sense throughout. A procedure is $(\eps,\delta)$-sound if, with probability at least
$1-\delta$, no false verdict is ever issued at any time. Soundness at every
stopping time is not optional in deployment, because promotion pipelines poll
the audit continuously and act the moment a verdict fires, which is precisely
the regime in which fixed-$n$ tests lose their guarantees
\citep{ramdas2023game}.

\section{The \discern{} Protocol}
\label{sec:protocol}

\discern{} (DISagreement-supported CERtificatioN)\footnote{The name is unrelated to the DISCERN instrument for appraising written consumer health information \citep{charnock1999discern}.} runs two tiers on the same
stream. Both tiers are driven by confidence sequences (CSs), which are
sequences of intervals $C_t$ with
$\Prob(\exists t : \theta \notin C_t)\le\delta$ for the
target parameter $\theta$ \citep{howard2021time}. We use the
predictable-plug-in empirical-Bernstein CS of \citet{waudby2024estimating},
whose radius adapts to the empirical variance of the increments. Its exact
form is recalled in Appendix~\ref{app:cs}.

\subsection{Tier 0, the zero-label certificate}
\label{sec:tier0}

The indicator stream $(A_t)$ is observed without labels. Maintain a CS
$[\,\underline\rho_t, \overline\rho_t\,]$ for $\rho$ at level $\delta_0$.
By Lemma~\ref{lem:support}, on the CS's coverage event,
\begin{equation}
\label{eq:tier0}
B\,\overline\rho_t \;<\; \eps
\quad\Longrightarrow\quad
|\Delta| \;<\; \eps ,
\end{equation}
so the auditor issues \textsc{Safe}$(\eps)$ having requested zero labels.
The firing threshold is the strict inequality $B\overline\rho_t<\eps$
precisely so the certificate implies the strict claim $\Delta<\eps$ that
\textsc{Safe}$(\eps)$ asserts, with no slippage at the boundary $\Delta=\eps$.
Tier 0 certifies exactly the updates that dominate production practice, namely
quantized builds, retrains that converge to near-identical decisions, and
small hyperparameter changes. When $\rho \le \eps/(2B)$, the certificate
fires after $O\!\big(B\log(1/\delta_0)/\eps\big)$ unlabeled points
(Proposition~\ref{prop:tier0}). The constant $B$ is a design quantity
rather than an estimate. For $0/1$ and other normalized losses $B=1$ by
construction, and for cost-sensitive losses it is the known maximum cost.
The certificate uses only the one-sided bound $\Delta\le B\rho$, and
overestimating $B$ degrades it linearly, since the certificate fires once
$\overline\rho_t<\eps/B$.

\subsection{Tier 1, auditing the disagreement region}
\label{sec:tier1}

When Tier 0 cannot fire, information must be bought, and
Lemma~\ref{lem:support} says to buy it only on disagreements. At each point
with $A_t=1$ the auditor requests the label with probability $\pi_t$, a
predictable routing rule. It may depend on everything observed
before $t$ and on side information about $X_t$ itself, including scores from
a learned judge model, subject only to $\pi_t \in [\pi_{\min}, 1]$. Routing
may never consult the incoming label $Y_t$. This is the independence
condition of Lemma~\ref{lem:unbiased}, and it is the one deployment
requirement. It is a genuine assumption, not a formality. A pipeline that
decides to label a point \emph{because} the outcome already looked bad,
for example queuing an example for review after a user complaint, makes
$L_t$ depend on $Y_t$ and biases the Horvitz--Thompson estimator, voiding
the guarantee. The safeguard is operational and simple, namely that the
routing rule must be a function of pre-label information only ($\filt_{t-1}$
and $X_t$, including judge scores), with the label drawn only after the
query decision is fixed. Under that discipline the condition holds by
construction. Let
$L_t\sim\mathrm{Bern}(\pi_t)$ be the labeling indicator. The
importance-weighted (Horvitz--Thompson) increment
\begin{equation}
\label{eq:ht}
Z_t \;=\; A_t \,\frac{L_t}{\pi_t}\, D_t \;\in\; [-c, c],
\qquad c := B/\pi_{\min},
\end{equation}
satisfies $\E[Z_t\mid \filt_{t-1}] = \Delta$ (Lemma~\ref{lem:unbiased}).
\discern{} maintains the empirical-Bernstein CS
$[\,\underline\Delta_t, \overline\Delta_t\,]$ for $\Delta$ at level
$\delta_1$ over the increments $Z_t$ (agreements contribute exact zeros and
consume no labels) and issues
\[
\textsc{Safe}(\eps) \text{ when } \overline\Delta_t < \eps,
\qquad
\textsc{Regression} \text{ when } \underline\Delta_t > 0 .
\]
Setting $\delta_0 + \delta_1 = \delta$ makes the combined procedure
$(\eps,\delta)$-sound (Theorem~\ref{thm:validity}). Three design points
deserve emphasis.

\paragraph{Judges accelerate, never invalidate.}
A learned judge (an auxiliary model scoring which disagreements look
suspicious) can drive $\pi_t$ to concentrate labels where regressions would
show. Because validity requires only predictability and the floor
$\pi_{\min}$, a judge that is stale, biased, or adversarial can waste budget
but cannot break coverage (Theorem~\ref{thm:routing}). This resolves the
central tension of LLM-as-judge evaluation. The judge is used for
efficiency, the guarantee never depends on its quality.

\paragraph{Per-slice verdicts.}
Promotion decisions are rarely global. Partition traffic into $K$ slices
(regions, languages, document types) and run one CS per slice with budget
$\delta_1/K$. Union coverage gives simultaneous validity of all $K$ verdicts
(Proposition~\ref{prop:slices}), enabling verdicts of the form ``safe
overall, regressing on slice $k$'' with localization guarantees.

\paragraph{Compute cost.}
\discern{} runs both models on audited traffic. This shadow-scoring cost is
standard in deployment (canary and shadow testing already do it) and is the
resource the protocol deliberately trades against labels: Theorem~\ref{thm:blind}
shows the trade buys a $1/\rho$ label reduction.

\section{Theory}
\label{sec:theory}

Throughout, $\Lb := \log(2/\delta_1)$ and $\widetilde O(\cdot)$ hides factors
polylogarithmic in $t$, $1/\eps$, and $\log(1/\delta)$. Proofs are in
Appendix~\ref{app:proofs}.

\subsection{Validity}

\begin{lemma}[Conditional unbiasedness]
\label{lem:unbiased}
Let $\pi_t$ be $\filt_{t-1}\vee\sigma(X_t)$-measurable with
$\pi_t\ge\pi_{\min}>0$, and let $L_t\mid \filt_{t-1}, X_t, Y_t \sim
\mathrm{Bern}(\pi_t)$ be drawn independently of $Y_t$ given $(\filt_{t-1},
X_t)$. Then the increment~\eqref{eq:ht} satisfies
$\E[Z_t \mid \filt_{t-1}] = \Delta$ and $|Z_t| \le B/\pi_{\min}$.
\end{lemma}

\begin{theorem}[Anytime validity]
\label{thm:validity}
Run Tier 0 at level $\delta_0$ and Tier 1 at level $\delta_1$ with any
predictable routing $\pi_t\in[\pi_{\min},1]$. Then with probability at least
$1-(\delta_0+\delta_1)$, simultaneously for all $t\ge1$:
$\rho\in[\underline\rho_t,\overline\rho_t]$ and
$\Delta\in[\underline\Delta_t,\overline\Delta_t]$. Consequently no false
\textsc{Safe}$(\eps)$ or \textsc{Regression} verdict is ever issued, at any
stopping time, under continuous monitoring.
\end{theorem}

\begin{theorem}[Routing robustness]
\label{thm:routing}
The conclusion of Theorem~\ref{thm:validity} holds for every predictable
routing rule, including rules computed from an arbitrary (possibly
adversarially wrong) judge model. Routing affects only the label budget and
the certification time, never coverage.
\end{theorem}

\subsection{Label-complexity upper bounds}

\begin{proposition}[Tier-0 certification time]
\label{prop:tier0}
If $B\rho \le \eps/2$, Tier 0 issues \textsc{Safe}$(\eps)$ after at most
$O\!\big((B/\eps)\log(1/\delta_0)\log\log(B/\eps)\big)$ unlabeled points,
with zero labels.
\end{proposition}

\begin{theorem}[Labels to a verdict]
\label{thm:upper}
Let $\pi_t \equiv \pi \in [\pi_{\min},1]$ and suppose the update is
$\eps$-benign, $\Delta \le \eps/2$. With probability at least $1-2\delta_1$,
Tier 1 issues \textsc{Safe}$(\eps)$ by stream time
\[
T^\star
\;=\;
\widetilde O\!\Big(
\frac{B^2\rho}{\pi\,\eps^2}\,\Lb
\;+\;
\frac{B}{\pi_{\min}\,\eps}\,\Lb
\Big),
\]
and the number of labels requested satisfies
\[
N_{\mathrm{lab}}
\;=\;
\widetilde O\!\Big(
\frac{B^2\rho^2}{\eps^2}\,\Lb
\;+\;
\frac{\pi}{\pi_{\min}}\cdot\frac{B\rho}{\eps}\,\Lb
\Big).
\]
Symmetrically, if $\Delta \ge 2\eps$ a \textsc{Regression} verdict fires
within the same bounds with $\eps$ replaced by $\Delta$.
\end{theorem}

The first term is the variance regime and dominates when $\rho \gtrsim \eps/B$,
the second is the range regime of the empirical-Bernstein boundary and
dominates for small $\rho$. Combined with Tier 0, which removes the regime
$\rho \le \eps/(2B)$ entirely, the audited tier is only ever invoked where
the quadratic term is active, and total label cost is
$\widetilde O(B^2\rho^2\Lb/\eps^2)$.

The same argument applied to the uniform-labeling baseline (label every
point, run the same variance-adaptive CS on the raw increments $D_t$, whose
variance is at most $B^2\rho_1$) gives the baseline cost
\begin{equation}
\label{eq:uniform}
N_{\mathrm{lab}}^{\mathrm{unif}}
\;=\;
\widetilde O\!\Big(\frac{B^2\rho}{\eps^2}\,\Lb + \frac{B}{\eps}\,\Lb\Big),
\end{equation}
so the predicted savings ratio is
$N_{\mathrm{lab}}/N_{\mathrm{lab}}^{\mathrm{unif}} = \Theta(\rho \vee \eps/B)$,
and the auditor pays only the disagreement fraction of the uniform bill. Our
experiments confirm both the individual rates and the ratio
(Section~\ref{sec:exp-savings}).

\subsection{Matching label-complexity lower bounds}

The next results show the protocol is not merely efficient but optimal, and
quantify exactly what the pairing signal is worth. Proofs use a
change-of-measure argument for adaptively stopped, adaptively routed
procedures in the style of sequential identification lower bounds
\citep{kaufmann2016complexity}, with a two-world construction in which
unlabeled observations are common to both worlds and only labeled
disagreements carry information. The change-of-measure machinery is
standard, and the two-atom construction is related to noise-rate lower
bounds in active learning \citep{hanneke2014theory}. What is new is not the
technique but the object it is applied to, a paired risk difference whose
signal lives on an observable measure-$\rho$ support, which is what produces
the $\rho$-dependence and, against a pairing-blind auditor, the $\Theta(1/\rho)$
separation. The construction does not reduce to substituting $\rho$ for a
noise rate in an existing bound, because the two worlds must share their
entire unlabeled stream, including the disagreement pattern, and differ only
in the labels on disagreements.

\begin{theorem}[Lower bound for adaptive auditors]
\label{thm:lower}
Fix $B=1$, $\eps\in(0,\tfrac18]$, and $\rho\in[2\eps,\tfrac12]$. There exist
distributions $P_0, P_1$, both with disagreement rate $\rho$, with
$\Delta(P_0)=0$ and $\Delta(P_1)=\eps$, whose unlabeled streams
$(X_t, f(X_t), g(X_t))$ are identically distributed, such that any
$(\eps,\delta)$-sound procedure that issues \textsc{Safe}$(\eps)$ with
probability at least $1-\delta$ under $P_0$ must request
\[
\E_{P_0}[N_{\mathrm{lab}}]
\;\ge\;
\frac{\rho^2}{4\,\eps^2}\,\log\!\frac{1}{2.4\,\delta}
\]
labels, regardless of how adaptively it routes labels or stops.
\end{theorem}

Together with Theorem~\ref{thm:upper}, the label complexity of certified
update auditing is settled at $\widetilde\Theta(\rho^2/\eps^2)$ in the
audited regime $\rho\ge2\eps$, and at zero labels in the Tier-0 regime
$\rho\le\eps/2$. The intermediate window $\rho\in(\eps/2,2\eps)$ is not a
gap in the characterization but the regime where the rate itself is
constant. There $\rho\asymp\eps$, so the upper bound
$\widetilde O(\rho^2/\eps^2+\rho/\eps)$ and the lower bound
$\Omega(\rho^2/\eps^2)$ are both $\widetilde\Theta(1)$, matching up to a
constant and log factors like everywhere else. What is unpinned in the band
is only the exact constant, which is standard for change-of-measure bounds
and does not affect the rate.

\begin{theorem}[Lower bound for pairing-blind auditors]
\label{thm:blind}
In the setting of Theorem~\ref{thm:lower}, call a procedure
pairing-blind if it does not evaluate the model pair on unlabeled
points. Its labeling decisions and verdict are measurable with respect to
its own past labeled outcomes alone. This class prices exactly the standard
practice of labeling sampled traffic without shadow-scoring it. Any
$(\eps,\delta)$-sound pairing-blind procedure must request
\[
\E[N_{\mathrm{lab}}]
\;\ge\;
\frac{\rho}{4\,\eps^2}\,\log\!\frac{1}{2.4\,\delta}
\]
labels. Uniform labeling with an empirical-Bernstein CS achieves this rate
up to logarithmic factors.
\end{theorem}

\begin{corollary}[The value of pairing]
\label{cor:value}
In the regime $\rho \ge 2\eps$, the ratio of optimal pairing-blind to
optimal pairing-aware label complexity is $\Theta(1/\rho)$. Observing both
models' predictions on unlabeled traffic, a pure compute cost, cuts the
human labeling bill by the inverse disagreement rate.
\end{corollary}

Corollary~\ref{cor:value} is the paper's economic statement. At the
disagreement rates of real update pairs, whose medians run from below
$10^{-3}$ to $0.15$ across update types (Table~\ref{tab:anatomy}) and reach
$0.45$ on individual pairs, the factor ranges from a few-fold at the
high-$\rho$ end to beyond $1000\times$ when $\rho$ is small.

\begin{remark}[What the pairing-blind class prices]
\label{rem:blind}
The separation prices exactly one capability, observing the pair's
agreement on unlabeled traffic. An auditor granted the disagreement
indicators $A_t$ on unlabeled points is by definition not pairing-blind. It
can restrict labeling to disagreements and attain the pairing-aware rate of
Theorem~\ref{thm:upper}. The class of Theorem~\ref{thm:blind} is therefore
not a tailored weakening but the natural dichotomy between auditors that
shadow-score all traffic and those that label sampled traffic without doing
so, the latter being standard when running both models on the full stream
is not instrumented. The $\Theta(1/\rho)$ gap is the operational value of
that instrumentation, a compute cost, measured in the labels, a human cost,
that it saves. The gap does not survive granting the blind auditor the
disagreement signal, and that is the point. The signal is precisely what
pairing is.
\end{remark}

\subsection{Per-slice certification}

\begin{proposition}[Simultaneous slice verdicts]
\label{prop:slices}
Partition $\mathcal X$ into slices $S_1,\dots,S_K$ and run the Tier-1 CS on
slice $k$ at level $\delta_1/K$ for the slice difference
$\Delta_k = \E[D \mid X \in S_k]$. Then with probability at least
$1-\delta_1$ all $K$ CSs cover simultaneously for all time, so every
per-slice \textsc{Safe} and \textsc{Regression} verdict is simultaneously
sound, and any slice with $\Delta_k \ge 2\eps_k$ is alarmed within the
slice-local bound of Theorem~\ref{thm:upper} applied to its own traffic
share.
\end{proposition}

\subsection{Sequences of promotions}
\label{sec:sequences}

Deployed systems are not updated once. A promoted candidate becomes the
next incumbent, a new candidate arrives, and the sequence of approvals
carries its own risk, the bio-creep phenomenon of
\citet{feng2021approval} in which repeated tolerant approvals let quality
drift downward. Our audits compose across such a sequence at no
additional modeling cost.

\begin{proposition}[Sequence-level soundness]
\label{prop:sequence}
Let updates arrive as a possibly adaptive sequence, where the $k$-th
incumbent-candidate pair and audit configuration may depend on all
earlier data and verdicts. Run audit $k$ on fresh traffic at levels
$(\delta_0^{(k)},\delta_1^{(k)})$ with
$\sum_{k\ge1}\big(\delta_0^{(k)}+\delta_1^{(k)}\big)\le\delta_{\mathrm{seq}}$.
Then with probability at least $1-\delta_{\mathrm{seq}}$, no audit in the
entire sequence ever issues a false verdict. In particular the number of
regressing updates that are ever certified \textsc{Safe} at their stated
tolerances is zero with probability at least $1-\delta_{\mathrm{seq}}$.
\end{proposition}

With the schedule $\delta^{(k)}\propto k^{-2}$ the guarantee covers an
unbounded stream of future updates from a single constant budget, and
with an equal split it covers a known horizon. Each successive audit is
more expensive by only a $\log(1/\delta^{(k)})$ factor
(Theorem~\ref{thm:upper}), so the price of lifetime soundness is
logarithmic in the update index. The alpha-investing refinement of
\citet{feng2021approval}, which grows the error budget after good
outcomes, ports directly and can improve power, but the union schedule
already rules out bio-creep in the family-wise sense.

Two measurability points make the argument precise. First, the $k$-th audit
may be chosen fully adaptively, because its configuration is
$H_k$-measurable and Theorem~\ref{thm:validity} is applied conditionally on
$H_k$; adaptivity of the promotion policy costs nothing. Second, the
proposition runs each audit on fresh traffic, which is what lets audit
$k$'s increments be conditionally valid given $H_k$. If instead audits
share or overlap traffic, family-wise soundness still holds, since the
union bound over $\sum_k\delta^{(k)}$ needs no independence between audits,
and only the label-efficiency of concurrent audits degrades, not their
validity.

\subsection{Beyond prediction-determined losses}
\label{sec:soft}

Assumption~\ref{ass:pred} makes disagreement a hard indicator. When the loss
reads the model's continuous output, a score or a logit, rather than only the
discrete prediction, agreement of predictions no longer forces $D_t=0$, and
the hard support identity fails. A soft version survives whenever the loss is
Lipschitz in the model output, which covers calibration losses, smooth
surrogate losses, and reward losses that are Lipschitz in a scalar score.

\begin{proposition}[Soft support identity]
\label{prop:soft}
Let $s_f(X),s_g(X)$ be the incumbent and candidate outputs in a normed space,
and suppose $\ell(\cdot,y)$ is $L$-Lipschitz in its first argument uniformly
in $y$, so $|\ell(a,y)-\ell(a',y)|\le L\,\|a-a'\|$. Define the observable soft
disagreement $\eta_t=\|s_g(X_t)-s_f(X_t)\|$, computable without labels. Then
\[
|D_t|\;\le\;L\,\eta_t \quad\text{a.s.,}\qquad
|\Delta|\;\le\;L\,\E[\eta_t],
\]
and $D_t=0$ wherever $\eta_t=0$. Consequently, running the confidence
sequence of Appendix~\ref{app:cs} on the unlabeled stream $(\eta_t)$ with
upper endpoint $\overline\eta_t$, the strict trigger $L\,\overline\eta_t<\eps$
certifies $|\Delta|<\eps$ with zero labels; and the importance-weighted
audited tier of Section~\ref{sec:tier1}, restricted to $\{\eta_t>0\}$, is
anytime-valid with conditional second moment
$\E[Z_t^2\mid\filt_{t-1}]\le L^2\,\E[\eta_t^2]/\pi_{\min}$, so an
$(\eps,\delta)$ verdict costs
$\widetilde O\!\big(L^2\E[\eta_t^2]\eps^{-2}+L\,\E[\eta_t]\eps^{-1}\big)$
labels.
\end{proposition}

\begin{proof}
Lipschitzness gives
$|D_t|=|\ell(s_g(X_t),Y_t)-\ell(s_f(X_t),Y_t)|\le L\|s_g(X_t)-s_f(X_t)\|
=L\eta_t$; take expectations for the bound on $|\Delta|$, and $\eta_t=0$
forces $s_f(X_t)=s_g(X_t)$ hence $D_t=0$. On the CS coverage event
$\E[\eta_t]\le\overline\eta_t$, so $|\Delta|\le L\E[\eta_t]\le
L\overline\eta_t<\eps$. Unbiasedness and time-uniform validity of the audited
tier are Lemma~\ref{lem:unbiased} and Theorem~\ref{thm:validity} verbatim,
since only predictability of the routing is used; the second-moment bound
follows from $|D_t|\le L\eta_t$ and $\E[L_t\mid\filt_{t-1},X_t]=\pi_t$, and
the label count is then the width lemma applied with $L^2\E[\eta_t^2]$ in
place of $B^2\rho$.
\end{proof}

The prediction-determined case is the special instance $\eta_t=A_t$, $L=B$,
which recovers $|\Delta|\le B\rho$ and the hard protocol. The soft identity
thus extends the zero-label certificate and the audited tier to
score-reading losses at the price of replacing the disagreement rate $\rho$
by the soft-disagreement moments $\E[\eta_t]$ and $\E[\eta_t^2]$. Matching the
lower bound in this soft regime, and handling open-ended judged generation
where the output is discrete text rather than a Lipschitz score, are left
open.

\subsection{Beyond stationarity}
\label{sec:drift}

Under distribution shift there is no fixed $\Delta$ to certify. Two honest
statements are available. First, the predictable-plug-in CS remains a valid
CS for the running weighted average of the conditional means
$\widetilde\Delta_t = \sum_{i\le t}\lambda_i\,\E[Z_i\mid\filt_{i-1}]
/\sum_{i\le t}\lambda_i$ \citep{waudby2024estimating}, so verdicts can be
read as statements about traffic-averaged regression over the audit window.
Second, for operators who need current-regime verdicts, \discern{} restarts
the CS on a sliding window, trading a $\log$ factor for freshness. We
report both sides empirically. The unrestarted CS under two aggressive
synthetic drift schedules miscovers a moving running-average target
($6.4\%$ under smooth drift, $15.9\%$ under adversarial bursts,
Table~\ref{tab:validity}), and the windowed variant repairs it on paired
same-seed streams ($0.0001$ and $0.065$ against $0.063$ and $0.156$
unrestarted, Section~\ref{sec:exp-drift}). We flag the unrestarted numbers
openly because silent miscoverage under drift is precisely the failure mode
a certified tool must not hide.

\section{Experiments}
\label{sec:experiments}

Our evaluation was pre-registered. Gates on validity, label efficiency,
zero-label coverage, and power were frozen, with pass thresholds, before the
main campaign ran, and a disjoint confirmation seed family was reserved
untouched until the final replication run (Appendix~\ref{app:prereg}). All
experiments replay real model pairs against real feature streams. Code,
per-stream records ($14{,}233$ stream-level JSON artifacts, Table~\ref{tab:accounting}), and one-command
reproduction scripts will be released publicly upon publication and are
available to reviewers on request during review (Appendix~\ref{app:repro}).

\subsection{Update-pair zoo}
\label{sec:zoo}

\paragraph{Feature-space pairs.}
We build $715$ update pairs across $13$ settings spanning four modalities:
RoBERTa text features (SST-2 sentiment), LayoutLMv3 document features
(the public Tobacco-3482 corpus and an internal financial-documents
classification corpus, the latter not part of the public release; the
document modality is anchored by the public Tobacco-3482 so no conclusion
rests on the internal set), supervised vision features
(ResNet-18 and ViT on CIFAR-10/100), and foundation-model features (DINOv2
and CLIP). In each setting a labeled audit pool ($25\%$ of the data,
disjoint from all fitting) provides exact ground truth $\Delta$ and $\rho$
for every pair, which is what makes measuring true miscoverage possible at
this scale. The incumbent is a classifier head trained on the remaining
data. Eleven update types produce the candidate, chosen so that the disagreement
rate and the true harm vary independently across the zoo
(Table~\ref{tab:anatomy}). Seven are benign, with true risk difference
essentially zero (retrain with a new seed and data order, a $90\%$ subsample
retrain, a $4\times$ regularization change, int8 and int4 weight
quantization, a refit on noised features, and a head refit on $65\%$ of
feature dimensions), and they span the disagreement range from
$\rho\approx0$ to $\rho\approx0.003$. Four are genuine regressions ($\Delta>0$),
namely retrains on labels corrupted at $10\%$ and at $30\%$ uniformly, a
targeted $40\%$ corruption on a fifth of the classes, and an aggressive
rank-halving of the head weights. The rank-halving is the instructive case.
It is an efficiency change of the same family as quantization, the kind an
engineer might assume is safe, yet unlike int8 and int4 quantization
(median $\Delta\approx0$) it degrades accuracy by a median of $14$ points,
and \discern{} alarms on it. Disagreement rate alone does not order updates
by harm, which is exactly why the promotion decision certifies $\Delta$
rather than reading off $\rho$. Five exploratory seeds per cell yield the
$715$ pairs, and each pair is replayed under several independent streams.
The $6{,}890$ main-battery audits decompose as $4{,}290$ validity streams
($715$ pairs times six routings, five i.i.d.\ and one adversarially routed),
$1{,}430$ drift streams ($715$ times two schedules, smooth mixture and
adversarial burst), and $1{,}170$ mechanism streams ($455$ tier-0 timing runs
on the benign types and $715$ label-allocation ablations).
Table~\ref{tab:accounting} itemizes these totals alongside the pilot,
confirmation, slice, and baseline campaigns.

\begin{table}[t]
\centering
\small
\begin{tabular}{llcc}
\toprule
Regime & Update type & Median $\rho$ & Median $\Delta$ \\
\midrule
\multirow{7}{*}{Benign ($\Delta\approx0$)}
 & retrain (new seed) & $0.0000$ & $0.0000$ \\
 & int8 quantization & $0.0000$ & $0.0000$ \\
 & hyperparameter change & $0.0000$ & $0.0000$ \\
 & feature-noise refit & $0.0004$ & $0.0000$ \\
 & data refresh ($90\%$) & $0.0005$ & $0.0000$ \\
 & int4 quantization & $0.0005$ & $0.0000$ \\
 & feature-subset refit & $0.0026$ & $0.0000$ \\
\midrule
\multirow{4}{*}{Regression ($\Delta>0$)}
 & $10\%$ label corruption & $0.0051$ & $0.0005$ \\
 & $30\%$ label corruption & $0.0152$ & $0.0084$ \\
 & targeted corruption & $0.0184$ & $0.0088$ \\
 & rank reduction & $0.1508$ & $0.1414$ \\
\bottomrule
\end{tabular}
\caption{Anatomy of the update zoo. Median disagreement rate and median true
risk difference by update type across $13$ settings and $5$ seeds
($65$ pairs, each replayed as $3$ i.i.d.\ streams, so $n=195$ audits per
type), placed by ground-truth harm. Realistic benign
updates cluster at $\rho$ near zero, exactly the regime Tier 0 certifies for
free, while their true harm is zero throughout. The regressions show that
$\rho$ does not order updates by harm. Rank halving disagrees at
$\rho=0.15$ and truly degrades accuracy ($\Delta=0.14$), whereas the
milder label corruptions disagree far less and carry a spectrum of harm,
much of it below the tolerance at their median. Feature-noise refits have
median $\rho$ near zero but range up to $0.05$ in the foundation-feature
settings. Detection is measured (Table~\ref{tab:validity}) on the injected
regressions with $\Delta\ge0.02$, the subset a promotion gate must catch.}
\label{tab:anatomy}
\end{table}

\paragraph{Language-model pairs.}
To confirm the protocol end to end on real update processes rather than
head refits, we additionally fine-tune four Pythia backbones from
$160$M to $1.4$B parameters \citep{biderman2023pythia} with LoRA adapters
\citep{hu2022lora} on SST-2 \citep{socher2013recursive} and AG News
\citep{zhang2015character}, building $70$ update pairs of five types (seed
refresh, data refresh, int8 quantization of the adapter, and two corruption
strengths). Ground truth comes from held-out labeled evaluation splits.
These pairs exhibit the larger training variance of real fine-tuning
($\rho$ up to $0.31$), precisely the stress test a promotion auditor faces.
The engine behaves identically across scales. On the $30$ pairs built on billion-parameter
backbones ($1$B and $1.4$B, $60$ streams) miscoverage is $0.017$, power $0.92$, and false alarms
$0$. Int8 adapter quantization certifies with a median of $4$ labels at every
scale, and the same Tier-0/audit/alarm split holds.

\subsection{Validity, power, and false alarms}
\label{sec:exp-validity}

\begin{table}[t]
\centering
\small
\begin{tabular}{lcccc}
\toprule
Campaign & Streams & Miscoverage & Power & False alarms \\
\midrule
Pre-registered pilot (frozen gates) & $2{,}352$ & $0.0021$ & $1.000$ & $0$ \\
Main battery (i.i.d.\ and adversarial) & $4{,}290$ & $0.0002$ & $0.986$ & $0$ of $2{,}406$ \\
Held-out confirmation (reserved seeds) & $2{,}145$ & $0.0009$ & $0.970$ & $0$ \\
LLM pairs, $160$M--$410$M (LoRA) & $80$ & $0.0250$ & $0.923$ & $0$ \\
LLM pairs, $1$B--$1.4$B (LoRA) & $60$ & $0.0167$ & $0.917$ & $0$ \\
\midrule
Drift, smooth mixture (reported) & $715$ & $0.064$ & -- & -- \\
Drift, adversarial bursts (reported) & $715$ & $0.159$ & -- & -- \\
\bottomrule
\end{tabular}
\caption{Validity and detection across all campaigns. Miscoverage is the
fraction of streams on which the $95\%$ CS ever excluded the true
$\Delta$ at any of up to $60{,}000$ time steps (time-uniform, so nominal is
$5\%$ for the whole trajectory). Power is the fraction of injected
regressions with $\Delta\ge0.02$ alarmed within $5{,}000$ stream points.
False alarms count \textsc{Regression} verdicts on truly benign pairs
($\Delta\le0$). Drift rows evaluate the unrestarted CS against a
running-average target outside its guarantee, quantifying
Section~\ref{sec:drift}.}
\label{tab:validity}
\end{table}

\begin{figure}[t]
\centering
\includegraphics[width=0.72\textwidth]{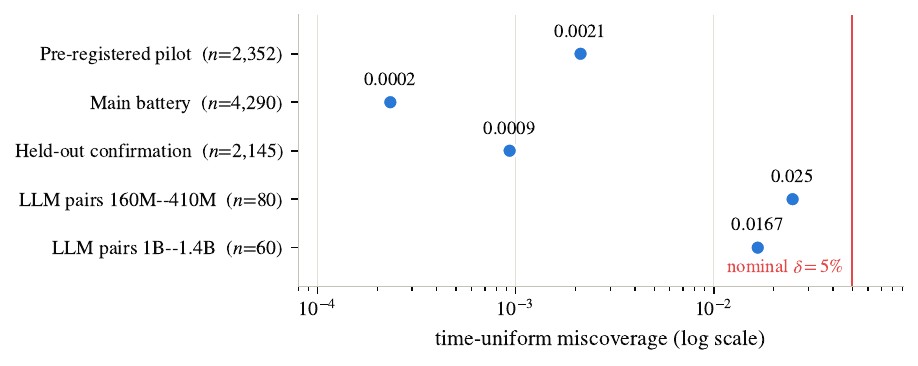}
\caption{Time-uniform miscoverage of the $95\%$ confidence sequence in
every campaign, against the nominal budget (red line, log scale). Each dot
is the fraction of that campaign's streams on which the CS ever excluded
the true $\Delta$ at any time step. Every campaign, including both
real-LLM zoos, sits below budget, most by one to two orders of magnitude.}
\label{fig:validity}
\end{figure}

\begin{figure}[t]
\centering
\includegraphics[width=0.78\textwidth]{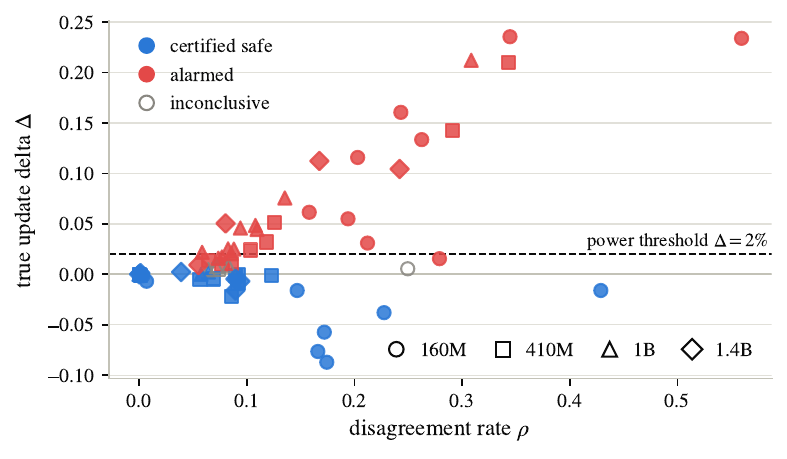}
\caption{Every language-model update pair ($70$ LoRA fine-tune pairs,
$160$M to $1.4$B parameters) by its true disagreement rate and true delta,
colored by the audit's verdict. Alarms (red) and safety certificates
(blue) separate along the true-delta axis, not the scale axis. Red points
below the dashed power threshold are still correct alarms (their
$\Delta>0$), the guarantee simply does not promise them. One audit was
inconclusive at its stream cap.}
\label{fig:llmscale}
\end{figure}

Table~\ref{tab:validity} summarizes (Figure~\ref{fig:validity} shows the
validity picture at a glance). Across $4{,}290$ i.i.d.\ and
adversarially routed streams in the main battery the time-uniform CS
excluded the truth on one stream, an empirical miscoverage of $0.0002$
against the nominal $5\%$ budget, and the guarantee holds with a wide margin, as
supermartingale boundaries are conservative at finite horizons. The same
margin means the boundary constants sit well above the
information-theoretic ones, so our optimality claims are at the rate
level, and measured label counts should be read accordingly. Injected
regressions with $\Delta\ge2\%$ are alarmed with power $0.986$ ($651$ of
$660$ injected streams) within
$5{,}000$ points (at the primary tolerance $\eps=1\%$, median alarm time
$964$ points, $90$th percentile $2{,}945$), and every one of the $9$ misses
sits within a factor $1.7$ of the tolerance boundary
($\Delta\in[0.020,0.033]$). No benign pair was ever
alarmed in the $2{,}406$ opportunities, that is, the main-battery streams
whose true delta is $\Delta\le0$ (the remaining $1{,}884$ of the $4{,}290$
have $\Delta>0$ and so are not false-alarm opportunities). The held-out confirmation, run once on
the reserved seed family after all development froze, reproduces every
number. Figure~\ref{fig:llmscale} shows every language-model pair, and
verdicts track the true delta across all four backbone scales. On the $40$
smaller-scale LoRA pairs the engine behaves identically.
Quantized adapters certify with a median of $24$ labels, seed and data
refreshes with genuinely positive $\Delta$ (real fine-tuning variance)
are correctly alarmed rather than certified, and strong corruptions alarm in
all $14$ of $16$ streams whose true $\Delta$ exceeded the tolerance
(the remaining two sat below it).

\subsection{Label efficiency and the two-tier ledger}
\label{sec:exp-savings}

\begin{figure}[t]
\centering
\includegraphics[width=\textwidth]{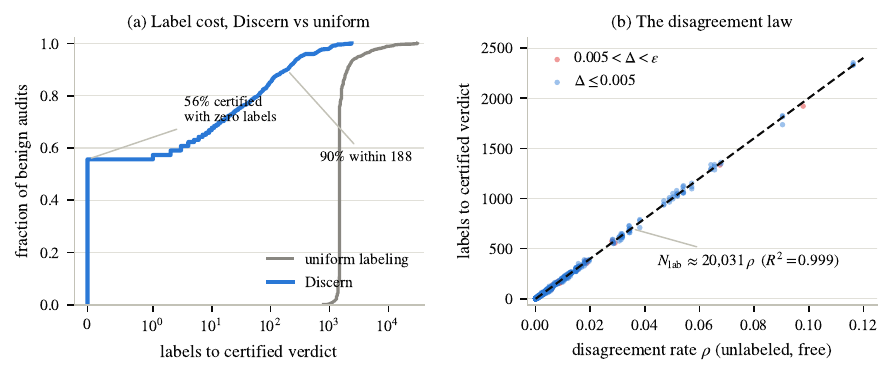}
\caption{(a) Distribution of labels to a certified verdict on the $1{,}672$
benign certifiable audits, \discern{} versus uniform labeling with the
identical CS machinery: $56\%$ certify with literally zero labels (Tier-0)
and $90\%$ within $188$, against a median of $2{,}906$ for uniform
labeling. (b) The disagreement law. Labels to a certified verdict against the pair's
true disagreement rate for every certifiable audit, colored by true delta
(blue $\Delta\le0.005$, red $0.005<\Delta<\eps$). All plotted pairs are
within the $\eps$ tolerance and correctly certified; the red points are
small genuine regressions that \textsc{Safe}$(\eps)$ promotes under its
noninferiority semantics. The cloud collapses onto the line
$N_{\mathrm{lab}} \approx \bar\pi\,T^\star(\eps)\,\rho$ with $R^2=0.999$,
the finite-horizon footprint of Theorem~\ref{thm:upper}'s label bound (at
$\eps=1\%$ the boundary's range regime makes $T^\star$ nearly
$\rho$-independent, so label spend is linear in $\rho$ with slope
$\bar\pi\,T^\star\approx 2.0\times10^4$).}
\label{fig:savings}
\end{figure}

The label ledger splits cleanly into the two tiers the theory predicts.

\paragraph{Tier 0 does most of the work for free.}
Three numbers measure the zero-label regime, at stated tolerances and denominators. At $\eps=1\%$, $56\%$ of all $1{,}672$ benign certifiable audits finish with exactly zero labels, and the $1{,}509$ audits with $\rho\le 1\%$ certify with a median of \textbf{zero} labels. In the pre-registered pilot at $\eps=2\%$, $98.2\%$ of benign pairs obtained a zero-label certificate within $20{,}000$ unlabeled points. This is the
deployment headline. The most common real updates (quantization, routine
retrains, hyperparameter changes) can be promoted with a certificate and no
labeling campaign at all.

\paragraph{The audited tier pays the predicted price.}
The $163$ benign audits with $\rho$ above the Tier-0 regime certify at a
median of $327$ labels against $2{,}906$ for uniform labeling with the same
CS machinery. Computed within each audit and then aggregated, the per-audit
label ratio has median $0.137$ (quartiles $0.079$ to $0.212$); the ratio of
the two marginal medians, $327/2{,}906=0.113$, is slightly smaller because
the two medians fall on different pairs. Either way the saving is a factor
of seven to nine, consistent with the $\Theta(\rho\vee\eps)$ ratio of
Theorem~\ref{thm:upper} versus~\eqref{eq:uniform}. The ablation endpoint
that labels every disagreement ($\pi=1$) spends a median of $671$ labels on
comparable pairs, confirming that subsampling disagreements, not merely
restricting to them, contributes materially. Figure~\ref{fig:savings}
(right) shows the label law. Across all certifiable benign audits, label
spend is linear in the pair's true $\rho$ with $R^2 = 0.999$, and the fitted
slope agrees with $\bar\pi\,T^\star(\eps)$ to within $2\%$, a direct
quantitative confirmation of the upper-bound mechanism. Tolerance scaling
matches the same bound. Tightening $\eps$ from $2\%$ to $1\%$ raises median
audited-tier spend from $600$ to $1{,}015$ labels on the moderate-$\rho$
band (details and the censoring protocol at the stream cap in
Appendix~\ref{app:eps}).

\subsection{Head-to-head label economics against sequential baselines}
\label{sec:exp-baselines}

No published sequential method targets the paired update delta, so we
construct the strongest sequential adaptations of the two adjacent tool
families and run all methods on identical streams with the identical
anytime-valid CS machinery and verdict rule, so the only degree of freedom
is where labels are spent. The prediction-powered baseline (PPI-style)
receives the pairing-aware surrogate for free on every point (zero on
agreements, a judge score on disagreements) and labels uniformly at rate
$q=\tfrac12$ with the unbiased control-variate increment
$\widehat D_t + (L_t/q)(D_t - \widehat D_t)$ and its tight range
$2/q-1$. The active-testing baseline
routes labels adaptively by the same judge on all points with floor $0.1$
and Horvitz--Thompson correction, but without the support restriction.
The baseline-comparison campaign has $312$ pairs, of which $275$ are benign
($246$ in the Tier-0 band and $29$ in the audited band) and $37$ are
regressions. Since label economics is about the cost of \emph{certifying}
a safe update, and regressions alarm rather than certify,
Table~\ref{tab:baselines} reports labels to a sound verdict on the $275$
benign pairs. The $37$ regressions are covered by the power and false-alarm
analysis of Section~\ref{sec:exp-validity}.

\begin{table}[t]
\centering
\small
\begin{tabular}{lrr}
\toprule
 & Tier-0 band & Audited band \\
 & ($\rho\le1\%$, $n{=}246$) & ($1\%{<}\rho\le20\%$, $n{=}29$) \\
\midrule
Uniform labeling & $1{,}476$ & $2{,}098$ \\
PPI-style (surrogate + uniform labels) & $2{,}232$ & $2{,}845$ \\
Active-testing (judge routing, all points) & $1{,}474$ & $1{,}948.5$ \\
\discern{} (support restriction) & $\mathbf{0}$ & $\mathbf{70}$ \\
\bottomrule
\end{tabular}
\caption{Median labels to a sound verdict at $\eps=1\%$, $\delta=0.05$, all
methods sharing the identical CS machinery and streams. Every method
reached a verdict on every pair (\discern{} on $28$ of $29$ audited-band
pairs) and no method false-alarmed anywhere. This cohort is the dedicated
comparison study (three exploratory seeds, one stream per pair, verdict
means certificate or alarm), so its medians differ from the battery
cohort of Section~\ref{sec:exp-savings}, which uses five seeds, several
streams per pair, and the certification endpoint. Neither adaptive routing
without the support restriction nor surrogate assistance with uniform
labeling escapes the pairing-blind cost of Theorem~\ref{thm:blind}. Routing
helps by at most a constant, and the control variate pays for its own
surrogate variance and enlarged increment range. Restriction to
disagreements is the mechanism, exactly as the
$\Theta(\rho)$ separation predicts.}
\label{tab:baselines}
\end{table}

\begin{figure}[t]
\centering
\includegraphics[width=\textwidth]{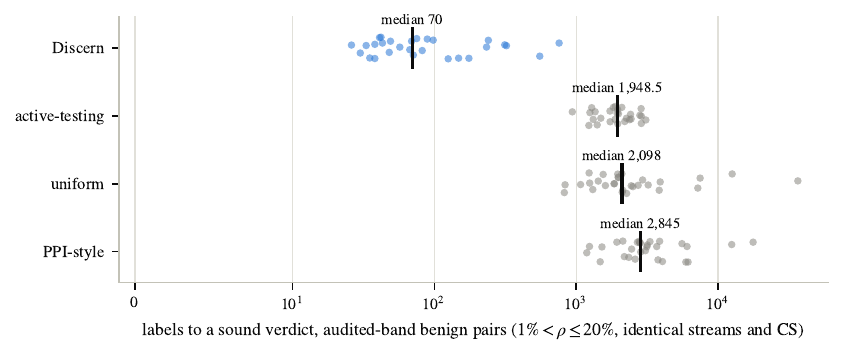}
\caption{The full per-pair distributions behind Table~\ref{tab:baselines}
(audited band, symlog scale, one dot per pair, black tick at the median).
Every \discern{} audit resolves left of every competitor audit, and the
separation is not a median artifact but holds pairwise across the entire
band.}
\label{fig:baselines}
\end{figure}

Two readings of Table~\ref{tab:baselines} matter (Figure~\ref{fig:baselines}
shows the full distributions). First, the gap between
\discern{} and the active-testing baseline ($28\times$ on the audited band)
isolates the support restriction from adaptivity in general. The active
baseline uses the same judge, the same HT correction, and the same CS, and
still pays the pairing-blind rate. Second, the PPI-style baseline is
slower than plain uniform labeling here, which is not a defect of
prediction-powered inference in its home setting (single-model means with
strong surrogates) but a structural mismatch. For the paired delta the
informative mass sits on a $\rho$-fraction of traffic, and no reweighting
of uniformly placed labels recovers what placement loses. This is not an
artifact of our instantiation. A sweep over the labeling rate and
surrogate quality, up to an unrealizable oracle surrogate, shows that the
best member of the PPI family on this problem is uniform labeling itself,
which PPI recovers exactly at $q=1$ (Appendix~\ref{app:ppisweep}).

\subsection{Repairing drift with a windowed confidence sequence}
\label{sec:exp-drift}

\begin{figure}[t]
\centering
\includegraphics[width=0.52\textwidth]{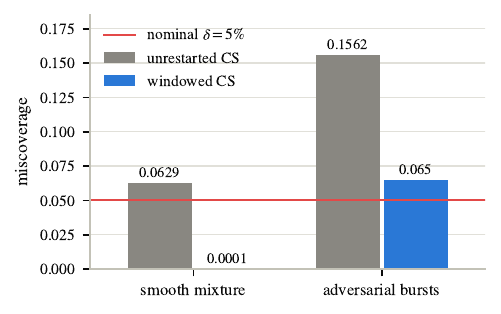}
\caption{Miscoverage under the two synthetic drift schedules, comparing the
unrestarted CS evaluated against a moving running-average target (outside
its guarantee) versus the windowed variant evaluated per window. The
windowed CS restores near-nominal calibration even under adversarial
bursts.}
\label{fig:drift}
\end{figure}

Section~\ref{sec:drift} scoped the guarantee under drift to weighted
averages and promised a windowed variant for current-regime verdicts. We
ran it (Figure~\ref{fig:drift}). The CS restarts every $2{,}000$ points,
each window covering its own window-average delta at level $\delta$. Across $429$ streams per
schedule ($8{,}580$ windows each). Under smooth-mixture drift, window
miscoverage is $0.0001$ versus $0.063$ for the unrestarted CS on the same
streams. Under adversarial burst drift, it is $0.065$ versus $0.156$. The windowed
variant restores near-nominal calibration even under bursts (residual
excess over $\delta=0.05$ comes from bursts shorter than the window, and
shrinking the window trades label efficiency for finer tracking). The
deployment guidance is concrete. Run the flat CS for the traffic-averaged
verdict and the windowed CS for regime-local alarms.

\subsection{Routing, floors, and confidence budgets}
\label{sec:exp-ablations}

Across a $480$-cell ablation grid (four settings, four update types, three
seeds), we find three things. First, routing driven by a sign-flipped,
adversarial judge leaves coverage untouched ($0$ violations in $48$ streams)
and changes label spend only, the empirical face of
Theorem~\ref{thm:routing}. Second, the floor
$\pi_{\min}\in\{0.1,0.25,0.5\}$ moves neither coverage nor median labels on
these benign pairs, matching the range-regime prediction that the floor
enters only through the boundary's lower-order term. Third, halving the
confidence budget ($\delta=0.05$ versus $0.10$) produces zero violations at
both levels. Across the grid, routing changed label spend only marginally relative to constant $\pi_t$. The measured savings come from the support restriction, and we therefore treat routing robustness as a safety property, namely that no judge can compromise validity, rather than as a measured speedup. Per-slice certification over $8$-way class-based slices on
$72$ multiclass audits localizes injected slice regressions with recall
$0.893$ at $3$ false-positive slices out of all benign slices audited, with
simultaneous coverage violated on no audit.

\section{Deployment and Governance Artifact}
\label{sec:deployment}

Each completed audit emits a compact evidence record, comprising the
tolerance and confidence, both tiers' final intervals, the verdict and its stopping time,
the number of labels consumed and their routing trace, and the per-slice
table. The record is sufficient for a third party to recompute the boundary
crossing from logged increments, making the promotion decision auditable
after the fact. Records compose across successive updates under the
budget schedule of Proposition~\ref{prop:sequence}, so a chain of
promotions carries one lifetime soundness guarantee. We present the regulatory connection as motivation rather than a compliance claim. The record is the kind of technical log contemplated by the logging and post-market monitoring duties for high-risk systems under the EU AI Act (Articles 12 and 72), and by internal model-risk-management review boards that must sign off on updates. A legal assessment is beyond our scope. We provide the emitter and a reference validation script in the
accompanying code release.

\section{Limitations and Scope}
\label{sec:limitations}

\textbf{Losses through internals.} Assumption~\ref{ass:pred} excludes losses
that consult scores rather than predictions (calibration error). For losses
Lipschitz in the model output, the soft support identity of
Proposition~\ref{prop:soft} recovers the zero-label certificate and the
audited tier with $\rho$ replaced by the soft-disagreement moments; matching
the lower bound there, and handling open-ended judged generation where the
output is discrete text, remain open. \textbf{Drift.} Our guarantee under drift is
the weighted-average statement of Section~\ref{sec:drift}, and the burst
experiment shows the unrestarted CS can miscover a moving target by design.
The windowed variant restores per-regime verdicts at a logarithmic price,
and choosing windows adaptively is future work. \textbf{Scale of the LLM
zoo.} The language-model pairs reach $1.4$B parameters, which fits our
audit-with-exact-ground-truth methodology. Nothing in the engine consumes
anything beyond paired predictions and labels, so scale enters only through
$\rho$ and $B$, and the identical behaviour from $160$M to $1.4$B supports
the claim. Auditing a $70$B model is a matter of shadow-scoring cost, not
new statistics. Auditing preference-tuned chat updates, where the loss
itself is a judged quantity, requires the internals extension above, and
our zoo contains no DPO or RLHF update pairs for the same reason. \textbf{Baseline
adaptations.} No off-the-shelf sequential method targets the paired update
delta, so our prediction-powered and active-testing comparisons
(Section~\ref{sec:exp-baselines}) are our own sequential adaptations of
those estimators, constructed generously (both receive the pairing-aware
surrogate for free) but still adaptations. \textbf{Allocation.} Per-slice
budgets use the uniform $\delta_1/K$ split. Optimal Neyman-style allocation
across slices is analyzed nowhere in this paper and left open.
\textbf{Sequences of promotions.} Proposition~\ref{prop:sequence} gives
family-wise soundness across a promotion chain from a union schedule.
What it does not give is adaptive budget allocation. The alpha-investing
policies of \citet{feng2021approval} spend error budget in response to
outcomes and can be strictly more powerful over long sequences, and
combining them with our per-audit label optimality is open.
\textbf{Estimand.} $\Delta$ certifies average regression, and per-slice
auditing covers known strata. Certifying against unknown subpopulations
connects to multicalibration and is a natural extension.

\section{Conclusion}

The promotion decision for a model update reduces to a paired estimand with
an observable support set, and taking that structure seriously yields a
complete theory, giving anytime-valid certificates robust to arbitrary label
routing, a zero-label tier that covers the bulk of real updates, matching
upper and lower bounds characterizing the label complexity at
$\widetilde\Theta(\rho^2/\eps^2)$ up to logarithmic factors, and a proven $\Theta(\rho)$ separation
that prices the value of running both models on unlabeled traffic. The
protocol certified real update pairs across four modalities and real LoRA
fine-tunes with essentially zero excess risk and a small fraction of the
labeling cost of current practice. We believe certified, rate-optimal
update auditing belongs in every promotion pipeline, and we will release
the complete machinery for that purpose upon publication.

\acks{Details of compute infrastructure and data provenance are in
Appendix~\ref{app:repro}. Feature caches are produced by standard pretrained
extractors and reused purely as fixed representations, as documented there.}

\appendix

\section{The Confidence Sequence}
\label{app:cs}

For increments $W_t\in[0,1]$ with conditional mean $\mu$, the
predictable-plug-in empirical-Bernstein (PrPl-EB) process of
\citet{waudby2024estimating} is, for a predictable $\lambda_t\in[0,\lambda_{\max}]$,
\begin{equation}
\label{eq:eb-mart}
M_t(m) \;=\; \prod_{i\le t}\exp\big(\lambda_i (W_i - m) - v_i\,\psi_E(\lambda_i)\big),
\qquad
\psi_E(\lambda) = -\log(1-\lambda)-\lambda,
\end{equation}
with $v_i = (W_i-\widehat\mu_{i-1})^2$ and $\widehat\mu_{i-1}$,
$\widehat\sigma^2_{i-1}$ the running mean and variance estimates. For
$m=\mu$, $M_t(\mu)$ is a nonnegative supermartingale
\citep[Thm.~2]{waudby2024estimating}, so by Ville's inequality
\citep{ville1939etude} the set
$C_t = \{m : M_t(m) < 2/\delta\}$, intersected over time, is a
$(1-\delta)$-CS for $\mu$. We use
$\lambda_t = \sqrt{2\log(2/\delta)/(\widehat\sigma_{t-1}^2\, t \log(t+1))}
\wedge \lambda_{\max}$. Increments in $[-c,c]$ are mapped affinely to
$[0,1]$ by $z\mapsto(z+c)/(2c)$, so the mapped increments have
conditional mean $(\Delta+c)/(2c)$, and $\lambda_t$ depends only on the
first $t-1$ increments, hence is $\filt_{t-1}$-measurable as the
guarantee requires. Tier 0 applies this to $W_t = A_t$, Tier 1 to
$W_t = (Z_t+c)/(2c)$.

The radius obeys, uniformly over $t$ with probability $1-\delta$,
\begin{equation}
\label{eq:width}
r_t \;=\;
O\!\Big(
\sqrt{\tfrac{\sigma^2\,(\log(2/\delta)+\log\log t)}{t}}
\;+\;
\tfrac{c\,(\log(2/\delta)+\log\log t)\,\log t}{t}
\Big),
\end{equation}
where $\sigma^2$ is the (conditional) variance scale of the increments. The
first term is the variance regime, the second the range regime. We use
\eqref{eq:width} as Lemma~\ref{lem:width} below.

\section{Proofs}
\label{app:proofs}

This appendix collects the proofs deferred from the main text, in the order
the results appear. Throughout we use the confidence-sequence facts recalled
in Appendix~\ref{app:cs}.

\subsection{Proof of Lemma~\ref{lem:unbiased}}
Condition on $\filt_{t-1}$ and $X_t$. Since $L_t$ is drawn
$\mathrm{Bern}(\pi_t)$ independently of $Y_t$ given $(\filt_{t-1},X_t)$, and
$\pi_t$ is measurable in the conditioning,
\[
\E\big[Z_t \mid \filt_{t-1}, X_t\big]
= A_t\,\frac{\E[L_t\mid\filt_{t-1},X_t]}{\pi_t}\,
\E\big[D_t \mid \filt_{t-1}, X_t\big]
= A_t\,\E[D_t \mid X_t].
\]
Towering over $X_t$ and using Lemma~\ref{lem:support} gives
$\E[Z_t\mid\filt_{t-1}] = \E[D_tA_t] = \Delta$. The range bound is
immediate from $|D_t|\le B$ and $\pi_t\ge\pi_{\min}$. \hfill$\blacksquare$

\subsection{Proof of Theorems~\ref{thm:validity} and \ref{thm:routing}}
The increments $A_t$ are i.i.d.\ $\mathrm{Bern}(\rho)$, so the Tier-0 CS
covers $\rho$ for all time with probability $1-\delta_0$ by the PrPl-EB
guarantee. For Tier 1, Lemma~\ref{lem:unbiased} shows
$(Z_t)$ has conditional mean $\Delta$ and range $[-c,c]$ for every
predictable routing rule, since predictability of $\pi_t$ is the only property
used, so an arbitrary judge, correct or adversarial, is covered. We stress
that the increments are neither independent nor identically distributed. The
validity we invoke is the martingale form of the PrPl-EB confidence sequence,
which requires only that the tracked quantity be the \emph{predictable
conditional mean} $\E[Z_t\mid\filt_{t-1}]=\Delta$ (constant here) and that
$\lambda_t$ be $\filt_{t-1}$-measurable, not that the $Z_t$ be i.i.d.\
(\citealp{waudby2024estimating}, Thm.~2; \citealp{howard2021time}). The mapped
increments $(Z_t+c)/(2c)$ therefore satisfy the conditions under which
\eqref{eq:eb-mart} is a supermartingale at the true mean, and Ville's
inequality gives time-uniform coverage at level $1-\delta_1$. A union bound
combines the tiers. Soundness of the verdicts follows since each verdict is
issued only when the relevant interval excludes the contradicted value, and
on the coverage event no interval ever excludes the truth.
\hfill$\blacksquare$

\subsection{Width lemma}
\begin{lemma}[CS radius]
\label{lem:width}
Let increments have conditional mean $\mu$, range $[0,1]$ after mapping, and
conditional variance at most $\sigma^2_{\mathrm{map}}$. With probability at
least $1-\delta$ the PrPl-EB CS radius satisfies \eqref{eq:width} with
$c=1$, and the empirical variance process concentrates so that the bound
holds with the true variance scale.
\end{lemma}
This is Theorem~2 together with the width analysis in Appendix~D of
\citet{waudby2024estimating}, and we do not reprove it. For our mapped
increments, $\sigma^2_{\mathrm{map}} \le \E[Z_t^2]/(2c)^2$, and
\begin{equation}
\label{eq:varZ}
\E[Z_t^2 \mid \filt_{t-1}]
= \E\Big[\frac{A_tD_t^2}{\pi_t}\,\Big|\,\filt_{t-1}\Big]
\le \frac{B^2\rho_1}{\pi_{\min}\wedge\pi_t} \le \frac{B^2\rho}{\pi},
\end{equation}
the last step for constant routing $\pi_t\equiv\pi$.

\subsection{Proof of Proposition~\ref{prop:tier0}}
Tier 0 tracks a Bernoulli($\rho$) mean with $\sigma^2 = \rho(1-\rho)\le\rho$
and range $1$. By Lemma~\ref{lem:width} the upper endpoint satisfies
$\overline\rho_t \le \rho + r_t$ with
$r_t = O(\sqrt{\rho\Lb'\!/t} + \Lb'\log t/t)$, $\Lb'=\log(2/\delta_0)+\log\log t$.
If $B\rho\le\eps/2$, the certificate fires once $B r_t< \eps/2$, which gives
$B\overline\rho_t\le B\rho+Br_t<\eps$ as the strict trigger requires. With
$\rho \le \eps/(2B)$ both terms are $O(\eps/B)$ once
$t = \Theta\big((B/\eps)\Lb'\log\log(B/\eps)\big)$
(the variance term needs $t \gtrsim B\rho\Lb'/\eps^2\cdot B \le \Lb' B/(2\eps)$,
the range term $t\gtrsim \Lb'\log t\, B/\eps$). \hfill$\blacksquare$

\subsection{Proof of Theorem~\ref{thm:upper}}
By \eqref{eq:varZ} the mapped variance scale is at most
$B^2\rho/(\pi(2c)^2)$. Lemma~\ref{lem:width}, undoing the affine map
(multiply radii by $2c$), gives uniformly with probability $1-\delta_1$
\[
r_t \;=\; O\!\Big(\sqrt{\tfrac{B^2\rho\,L_{\delta,t}}{\pi\,t}}
\;+\; \tfrac{B}{\pi_{\min}}\cdot\tfrac{L_{\delta,t}\log t}{t}\Big),
\qquad L_{\delta,t} := \Lb + \log\log t .
\]
If $\Delta\le\eps/2$, the verdict $\overline\Delta_t<\eps$ fires as soon as
$r_t\le\eps/2$ (on the coverage event $\overline\Delta_t \le \Delta + 2r_t$
by construction of the running intersection). Solving $r_t\le\eps/2$ for
$t$ yields
$T^\star = \widetilde O\big(B^2\rho\Lb/(\pi\eps^2) + B\Lb/(\pi_{\min}\eps)\big)$,
the first claim. Labels are requested only on sampled disagreements:
$N_{\mathrm{lab}} = \sum_{t\le T^\star} A_tL_t$ is stochastically dominated
by $\mathrm{Bin}(T^\star, \rho\pi)$, and a Chernoff bound gives
$N_{\mathrm{lab}} \le 2\rho\pi T^\star + O(\Lb)$ with probability
$1-\delta_1$. Multiplying out,
\[
\rho\pi\,T^\star
= \widetilde O\!\Big(\frac{B^2\rho^2}{\eps^2}\Lb
+ \frac{\pi}{\pi_{\min}}\cdot\frac{B\rho}{\eps}\Lb\Big).
\]
The regression side is symmetric with the roles of the endpoints exchanged.
The uniform baseline bound \eqref{eq:uniform} is the same computation with
$\pi=\pi_{\min}=1$, increments $D_t$ of variance
$\E[D^2]=\E[D^2\ind{D\neq0}]\le B^2\rho_1$, and every point labeled, so
labels equal stream time. \hfill$\blacksquare$

\subsection{Proof of Theorem~\ref{thm:lower}}
\label{app:lower}
Construction. Let $\mathcal X = \{a, d\}$ with
$\Prob(X{=}d)=\rho$. On the agreement atom $a$: $f=g$, both correct,
$D=0$, identically under both worlds. On the disagreement atom $d$: $f\neq g$
and the label makes exactly one of them correct, so $D = \pm1$ ($0/1$ loss,
$B=1$). Under $P_0$: $\Prob(D{=}{+}1\mid d) = 1/2$, hence $\Delta(P_0)=0$.
Under $P_1$: $\Prob(D{=}{+}1\mid d) = (1+\theta)/2$ with
$\theta := \eps/\rho \in (0,\tfrac12]$, hence $\Delta(P_1) = \rho\theta =
\eps$. Unlabeled observations (the atom identity and both predictions) have
identical laws under $P_0$ and $P_1$, and a label on atom $a$ reveals a
constant. Only a label on $d$ is informative, revealing a
$\mathrm{Bern}((1{+}\theta)/2)$ draw versus $\mathrm{Bern}(1/2)$.

Change of measure. Consider any procedure adapted to its own
observations (it chooses $\pi_t$ from the past, observes labels it pays
for, and stops at a stopping time $\tau$, issuing a verdict). Let
$N_d(\tau)$ be the number of labeled disagreement points at stopping. By
the sequential change-of-measure identity
\citep[Lemma~1]{kaufmann2016complexity}, for any stopping time $\tau$ and
any event $E\in\filt_\tau$,
\[
\E_{P_0}[N_d(\tau)]\;\kl\!\Big(\tfrac12,\tfrac{1+\theta}{2}\Big)
\;\ge\;
\kl\big(\Prob_{P_0}(E),\,\Prob_{P_1}(E)\big),
\]
because per-label KL from $P_0$ to $P_1$ is
$\kl(\tfrac12,\tfrac{1+\theta}{2})$ and all other observations contribute
zero KL. Take $E=\{\text{procedure issues \textsc{Safe}}(\eps)\}$. Soundness
under $P_1$ (where $\Delta=\eps$, so \textsc{Safe}$(\eps)$, the assertion
$\Delta<\eps$, is false) forces $\Prob_{P_1}(E)\le\delta$, while the
hypothesis gives $\Prob_{P_0}(E)\ge1-\delta$. Then
$\kl(1-\delta,\delta)\ge\log\frac{1}{2.4\delta}$
\citep{kaufmann2016complexity}. Finally, for $\theta\le\tfrac12$ the
per-label information is small:
\[
\kl\!\Big(\tfrac12,\tfrac{1+\theta}{2}\Big)
\;=\;
\tfrac12\log\tfrac{1}{1-\theta^{2}}
\;\le\; \theta^{2}
\qquad (\theta \le \tfrac12).
\]
Combining the three displays and substituting $\theta=\eps/\rho$,
\[
\E_{P_0}[N_{\mathrm{lab}}] \;\ge\; \E_{P_0}[N_d(\tau)]
\;\ge\; \frac{\kl(1-\delta,\delta)}{\kl(\tfrac12,\tfrac{1+\theta}{2})}
\;\ge\; \frac{1}{\theta^{2}}\,\log\!\frac{1}{2.4\,\delta}
\;=\; \frac{\rho^{2}}{\eps^{2}}\,\log\!\frac{1}{2.4\,\delta},
\]
and the factor $4$ in the theorem statement absorbs the slack of the
$\kl\le\theta^2$ step at the boundary $\theta=\tfrac12$.
\hfill$\blacksquare$

\subsection{Proof of Theorem~\ref{thm:blind}}
Same two worlds. Blindness matters here because $f$ and $g$ are
deterministic. A procedure that shadow-scores an unlabeled point learns its
atom for free and exits the class, which is precisely why the class models
labeling without paired inference. A pairing-blind procedure decides to
label a stream point without seeing the atom, so each label reveals a draw
of $D$ with law
$(1-\rho)\delta_0 + \rho\,\mathrm{Bern}^{\pm}((1{\pm}\theta)/2)$
under each world (mass $1-\rho$ at $0$, and $\pm1$ with probabilities
$\rho(1\pm\theta)/2$ under $P_1$, $\rho/2$ each under $P_0$). The shared
mass at $D=0$ contributes zero, so the per-label KL is
\[
\KL_{\mathrm{blind}}
\;=\;
\frac{\rho}{2}\log\frac{1}{1-\theta}
\;+\;
\frac{\rho}{2}\log\frac{1}{1+\theta}
\;=\;
\frac{\rho}{2}\,\log\frac{1}{1-\theta^{2}}
\;\le\; \rho\,\theta^{2}
\qquad (\theta\le\tfrac12).
\]
The same change-of-measure argument then yields
\[
\E[N_{\mathrm{lab}}]
\;\ge\; \frac{\log(1/2.4\delta)}{\rho\,\theta^{2}}
\;=\; \frac{\rho}{\eps^{2}}\,\log\!\frac{1}{2.4\,\delta},
\]
with the stated constant. For achievability,
uniform labeling observes exactly these draws, and the empirical-Bernstein
CS on them has variance scale $\E[D^2]=\rho$, so by the argument of
Theorem~\ref{thm:upper} with $\pi=1$ on all points it certifies within
$\widetilde O(\rho\Lb/\eps^2 + \Lb/\eps)$ labels. \hfill$\blacksquare$

\subsection{Proof of Proposition~\ref{prop:slices}}
Each slice CS is a valid time-uniform CS at level $\delta_1/K$ for its
conditional difference by Theorem~\ref{thm:validity} applied to the
thinned stream of slice-$k$ points (thinning preserves the conditional
mean and predictability). A union bound over $K$ slices gives simultaneous
coverage $1-\delta_1$. The detection bound is Theorem~\ref{thm:upper}
applied to the slice's own stream, whose length by time $t$ concentrates
around $t\,\Prob(X\in S_k)$. \hfill$\blacksquare$

\subsection{Proof of Proposition~\ref{prop:sequence}}
Let $H_k$ denote the history available when audit $k$ starts. It
determines the $k$-th model pair, tolerances, and budgets. The
supermartingale construction behind Theorem~\ref{thm:validity} uses only
the conditional distribution of audit $k$'s own increments given its own
filtration initialized at $H_k$, so Theorem~\ref{thm:validity} applies
conditionally, and the event $E_k$ that audit $k$ ever issues a false
verdict satisfies
$\Prob(E_k \mid H_k)\le\delta_0^{(k)}+\delta_1^{(k)}$ almost surely.
Taking expectations, $\Prob(E_k)\le\delta_0^{(k)}+\delta_1^{(k)}$, and a
union bound over $k$ gives
$\Prob\big(\bigcup_k E_k\big)\le\delta_{\mathrm{seq}}$. On the
complement, every issued \textsc{Safe} certificate is true at its stated
tolerance, so no regressing update is ever falsely certified.
\hfill$\blacksquare$

\section{Pre-registration and Seed Discipline}
\label{app:prereg}

Four gates were frozen before the main campaign. They are time-uniform miscoverage
at most $7\%$ (mandatory), median label ratio versus uniform at most $0.35$
(mandatory), zero-label certification of at least half the benign pairs
(reported), and power at least $95\%$ with zero false alarms (mandatory).
The pilot passed all four ($0.21\%$, ${\approx}0$, $98.2\%$,
$100\%$ with $0$ false alarms). Two pre-launch amendments were logged
before any gate data existed. They scope the validity gate to i.i.d.\ and
adversarially routed streams (drift reported separately, as the estimand
changes under drift), and add implementation notes on convex-head refresh
pairs. Exploratory seeds $\{2011, 3019, 4021, 5011, 6011\}$ were used for
all development. The confirmation family
$\{179424691, 198491329, 217645199, 236887699, 256203221\}$ was reserved
untouched and spent exactly once on the final replication row of
Table~\ref{tab:validity}.

\section{Tolerance Scaling and Censoring}
\label{app:eps}

Streams are capped at $30{,}000$/$40{,}000$/$60{,}000$ points for
$\eps = 2\%/1\%/0.5\%$. At $\eps=0.5\%$ a minority of moderate-$\rho$
audits certify within the cap (the required $T^\star$ exceeds it), so
median label spend at that tolerance is computed over certifying audits
only and reported as censored. Observed medians on the moderate-$\rho$
band: $600$ labels at $\eps=2\%$, $1{,}015$ at $1\%$, $1{,}012$
(censored) at $0.5\%$. The $2\%\to1\%$ step matches the range-regime
prediction of Lemma~\ref{lem:width} at these $\rho$ values. Nothing in the
paper's claims depends on the censored cell.

\section{Baseline Instantiation Sweep}
\label{app:ppisweep}

Table~\ref{tab:baselines} instantiates the PPI-style baseline at
$q=\tfrac12$ with the paper's margin-based surrogate. To rule out an
unfavorable instantiation, we sweep the labeling rate
$q\in\{0.1, 0.25, 0.5, 1\}$ and surrogate quality up to an unrealizable
oracle ($\widehat D_t = D_t$ exactly), on the same audited-band cohort and
streams, with the tight increment range $2/q-1$
(Table~\ref{tab:ppisweep}). Two facts anchor the sweep. At $q=1$ the
correction reconstructs $D_t$ exactly, so PPI coincides with uniform
labeling for any surrogate, analytically and in the table. For $q<1$ the
importance correction enlarges the increment range faster than
subsampling saves labels, so every swept configuration, including the
oracle, is at or above uniform's label cost. The base estimator is the
prediction-powered inference of \citet{angelopoulos2023prediction}. We also
sweep the power-tuning parameter $\lambda\in\{0.25,0.5,0.75\}$ of PPI++
\citep{angelopoulos2023ppipp}, which interpolates between the classical
estimator ($\lambda=0$) and full PPI ($\lambda=1$), on the same cohort
(Table~\ref{tab:ppippsweep}). Every $(\lambda,q)$ cell lands between
$2{,}236$ and $3{,}353$ median labels, again at or above uniform's $2{,}098$
and far above \discern{}'s $70$, and at $q=1$ each collapses to uniform by
construction. The active label-allocation method of
\citet{zrnic2024active} optimizes a fixed-horizon budget with the same
prediction-powered ingredients, so it inherits the same barrier on this
paired estimand. The best member of the prediction-powered
family on this problem is uniform labeling itself, and the comparison in
Table~\ref{tab:baselines} is therefore conservative in \discern{}'s
favor rather than against it.

\begin{table}[t]
\centering
\small
\begin{tabular}{lrrrr}
\toprule
Median labels to verdict & $q=0.1$ & $q=0.25$ & $q=0.5$ & $q=1$ \\
\midrule
Margin surrogate & $2{,}695^{*}$ & $3{,}562^{*}$ & $3{,}058^{*}$ & $2{,}098$ \\
Oracle surrogate ($\widehat D=D$) & $3{,}071^{*}$ & $2{,}947$ & $2{,}560$ & $2{,}098$ \\
\bottomrule
\end{tabular}
\caption{PPI-style baseline over labeling rate and surrogate quality on
the audited-band cohort of Table~\ref{tab:baselines} ($n=29$ pairs,
identical streams, tight range $2/q-1$). Asterisks mark configurations
where some audits reached no verdict within the stream cap, so the
median is over verdicts reached ($15$, $25$, $28$, and $25$ of $29$ for
the starred cells, left to right and top to bottom). Both rows equal
uniform labeling at $q=1$ by construction. The $q=\tfrac12$ margin cell
differs from Table~\ref{tab:baselines}'s instantiation only in the
labeling randomness. \discern{} certifies the same cohort at a median of
$70$ labels.}
\label{tab:ppisweep}
\end{table}

\begin{table}[t]
\centering
\small
\begin{tabular}{lrrr}
\toprule
Median labels to verdict & $q=0.1$ & $q=0.25$ & $q=0.5$ \\
\midrule
$\lambda=0.25$ & $2{,}501^{*}$ & $2{,}445^{*}$ & $2{,}236^{*}$ \\
$\lambda=0.5$ & $2{,}486^{*}$ & $2{,}451^{*}$ & $2{,}483^{*}$ \\
$\lambda=0.75$ & $2{,}725^{*}$ & $2{,}739^{*}$ & $3{,}353$ \\
\bottomrule
\end{tabular}
\caption{PPI++ power-tuning sweep \citep{angelopoulos2023ppipp} over the
mixing parameter $\lambda$ and labeling rate $q$, same audited-band cohort
and streams as Table~\ref{tab:ppisweep} ($n=29$ pairs). The column $q=1$ is
omitted because every $\lambda$ collapses to uniform labeling ($2{,}098$)
there. Asterisks mark cells where some audits reached no verdict within the
stream cap; verdicts reached (of $29$) are $21,27,28$ for $\lambda=0.25$,
$20,24,28$ for $\lambda=0.5$, and $15,27,29$ for $\lambda=0.75$, left to
right. No configuration beats uniform, so power tuning does not rescue the
prediction-powered family on this paired estimand. \discern{} certifies
the same cohort at a median of $70$ labels.}
\label{tab:ppippsweep}
\end{table}

\section{Reproducibility}
\label{app:repro}

All experiments are CPU-replayable except the LoRA fine-tunes (a single
commodity GPU, about three GPU-hours total). The complete pipeline is orchestrated by a
single resumable script with per-audit atomic artifacts, frozen-gate
evaluation, figure generation, and a final status report. Feature caches
(RoBERTa, LayoutLMv3, ResNet, ViT, DINOv2, CLIP) are produced by standard
pretrained extractors run once and reused here as fixed representations, so
no result depends on how they were computed. The release comprises the audit engine
(${\sim}300$ lines for both tiers), the update-pair builders, all
$14{,}233$ per-stream records itemized in Table~\ref{tab:accounting}, and one-command reproduction of every table
and figure. The pre-registration document with its frozen gates and timestamped amendments, and the reserved-seed protocol, are part of the same artifact. The full package will be released publicly
upon publication and is available to reviewers on request.

\begin{table}[H]
\centering
\small
\begin{tabular}{lrr}
\toprule
Campaign & Update pairs & Replayed streams \\
\midrule
Pre-registered pilot (frozen gates) & $455$ & $3{,}336$ \\
Main battery & $715$ & $6{,}890$ \\
Held-out confirmation (reserved seeds) & $715$ & $2{,}145$ \\
LLM zoo, $160$M--$410$M & $40$ & $80$ \\
LLM zoo, $1$B--$1.4$B & $30$ & $60$ \\
Ablation grid & -- & $480$ \\
Per-slice study & -- & $72$ \\
Baseline comparison & $312$ & $312$ \\
Windowed-drift study & -- & $858$ \\
\midrule
Total & & $14{,}233$ \\
\bottomrule
\end{tabular}
\caption{Authoritative accounting of every replayed audit stream reported in this paper. Rows are disjoint campaigns. Abstract-level summaries round down to $14{,}000$+.}
\label{tab:accounting}
\end{table}

\bibliography{references}

\end{document}